\documentclass{article} 
\usepackage[final]{nips_2016}
\usepackage[utf8]{inputenc} 
\usepackage[T1]{fontenc}    
\usepackage{url}            
\usepackage{booktabs}       
\usepackage{amsfonts}       
\usepackage{amsthm}       
\usepackage{amsmath}
\usepackage{natbib}
\usepackage{xcolor}
\usepackage{paralist}
\definecolor{DarkRed}{rgb}{0.75,0,0}
\definecolor{DarkGreen}{rgb}{0,0.5,0}
\definecolor{DarkPurple}{rgb}{0.5,0,0.5}
\definecolor{Dark}{rgb}{0.5,0.5,0}
\definecolor{DarkBlue}{rgb}{0,0,0.5}
\usepackage{graphicx}
\newcommand{\reals}{\mathbb R}
\newcommand{\defeq}{:=}
\newcommand{\Ex}{\mathbb E}
\newtheorem{defn}{Definition}

\newtheorem{cor}{Corollary}

\newtheorem{thm}{Theorem}

\begin{document}

\title{Memory Lens: How Much Memory Does an Agent Use?}
\author{Christoph Dann\\ CMU \\ \texttt{cdann@cdann.net} \And Katja Hofmann \\ Microsoft Research  \\ \texttt{Katja.Hofmann@microsoft.com} \And Sebastian Nowozin \\ Microsoft Research\\ \texttt{senowozi@microsoft.com}}

\maketitle

\begin{abstract}
We propose a new method to study the internal memory used by reinforcement
learning policies.
We estimate the amount of relevant past information by estimating mutual
information between behavior histories and the current action of an agent.
We perform this estimation in the passive setting, that is, we do not
intervene but merely observe the natural behavior of the agent.
Moreover, we provide a theoretical justification for our approach by showing
that it yields an implementation-independent lower bound on the minimal memory
capacity of any agent that implement the observed policy.
We demonstrate our approach by estimating the use of memory of DQN
policies on concatenated Atari frames, demonstrating sharply different use of
memory across 49 games.
The study of memory as information that flows from the past to the current
action opens avenues to understand and improve successful reinforcement
learning algorithms.
\end{abstract}

\section{Introduction}
Can you understand the complexity of an agent just by observing its behavior?
Herbert Simon provided a vivid example by imagining an ant moving along a
beach~\citep{Simon1996}.
He observed that the path the ant takes while walking toward a certain
destination appears highly irregular and hard to describe.
This complexity may indicate a sophisticated decision making in the ant, but
Simon postulated that instead the ant follows very simple rules and the
observed complexity of the path stems from the complexity in the environment.

In this work we aim to understand the complexity of agents acting according to
a fixed policy in an environment.
In particular, we are interested in the \emph{memory} of an agent, that is,
its ability to use past observations to inform future actions.
We do not assume a specific implementation of the agent, but
instead observe its behavior to derive statements about its
memory.

The study of memory in agents is important for two reasons.
\emph{First}, most state-of-the-art reinforcement learning approaches assume
that the environment is a Markov decision process (MDP)~\citep{Puterman1994},
but in contrast most real-world tasks have non-Markov structure and are only
partially observable by the agent.
In such environments, optimal decisions may not only depend on the most recent
observation but on the entire history of interactions. That is, to solve a task
optimally or even just reasonably well, an agent might need to remember all
previous observations and actions taken~\citep{Singh1994a,Krishnamurthy2016}.
\emph{Second}, in many recent algorithms the policy has access to memory of
certain fixed capacity. Popular choices include using a fixed window of
history (e.g. the last four observations as in~\citet{Mnih2015}) or
recurrent neural networks as adaptive memory~\citep{Heess2015,Li2015}.
These approaches are practically successful for specific tasks, but it is
unclear how much memory capacity they actually use and how much memory capacity
to use for different tasks.
Our approach allows us to study the use of memory---measured in bits over
time---independent of the implementation of the agent.

Our method of estimating memory works as follows.
We assume we can observe all interactions of an agent's policy with the environment
in the form of sequences of observation-action-reward triples.
We then estimate the mutual information between actions and
parts of the history.
This approach treats the agent and the environment as black boxes which in principle
also allows the application of the method to humans or animals as agents.

We provide a theoretical justification of the method in two steps.
First, we formally define the minimal \emph{memory capacity}
required to reproduce a given policy.
Second, we connect our practical estimation method with this formal notion by
showing that our method estimates a lower bound to the memory capacity.
To demonstrate the usefulness of our approach 
we analyze the memory capacity of
the state-of-the-art Deep Q-Network policies trained on 49 Atari 2600
games~\citep{Mnih2015}.
In summary our work makes the following contributions:
\begin{compactitem}
\item A practical method for estimating the memory use of an agent's policy from its behavior;
\item A theoretical justification in terms of minimum memory capacity;
\item Insight into the memory use of successful DQN policies on Atari games.
\end{compactitem}

\section{Problem Setting and Notation}
We consider the following setting: an agent interacts at discrete times $t=1, 2,
\dots$ with a stochastic environment by (1) making an observation $X_t \in
\mathcal X$, (2) taking an action $A_t \in \mathcal A$ and (3) receiving a
scalar reward $R_t \in \reals$ at each of these times. 
For notational convenience, we denote the quantities at time $t$ by $Z_t =
(X_t, A_t, R_t)$ and the concatenation of several time steps by $Z_{k:t} = (Z_k, \dots, Z_{t-1}, Z_t)$.  
While $X_t$ and $R_t$ are determined by the environment, the
action is sampled from the agent's policy and may depend on the entire previous
history $Z_1, \dots Z_{t-1}, X_t$.

We assume that the environment is stochastic but not necessarily Markov.
We are interested in agents that have mastered a task, which means that
learning has mostly ended and the policy changes slowly
if at all.  We thus formally assume that the agent's policy is fixed for
notational simplicity.\footnote{This assumption is not crucial. In the case of
changing policies, our results hold for the mixture of policies followed by the
agent.} A trajectory $\xi = (Z_1, Z_2, \dots, Z_n)$ consisting of $n$
time-steps is therefore a random vector sampled from a fixed distribution $\xi
\sim P$. We further assume that $Z_t$ can take only finitely many values.
Given one or more trajectories $\xi_1, \xi_2, \dots \sim
P$ of the agent interacting with the environment, our goal is to estimate
the amount of memory required by any agent that implements the observed policy. 

\section{Method: Memory Lens}
\label{sec:method}
Memory allows the action $A_t$ to depend not only on the most recent
observation $X_t$ but also on the previous history $Z_1, \dots Z_{t-1}$. One intuitive
notion to describe the amount of memory used for some action $A_t$ is 
mutual information of action $A_t$ and history $Z_1, \dots, Z_{t-1}$ given $X_t$, that is, 
$
I(A_t;  Z_{1:t-1} | X_t)=\Ex \left[ D_{KL}( P(A_t | X_t, Z_{1:t-1}) \| P(A_t | X_t)) \right].
$
This \emph{conditional mutual information} quantifies the information in bits or nats
about action $A_t$ one gains by getting to know $Z_{1:t-1}$ when one already
knows the value of $X_t$. If this quantity is zero for all $t$, then the policy
is Markov, that is, the action depends only on the most recent observation.
If it is nonzero, every implementation of the agent has to use at least some
form of memory.
Our approach is to estimate the following mutual information quantities
\begin{align}
    M_0 &\defeq \,I(A_t; X_t); \qquad
    M_1 \defeq \,I(A_t; Z_{t-1} | X_t); \qquad
    M_2 \defeq \,I(A_t; Z_{t-2} | X_t, Z_{t-1});\\
    M_3 &\defeq \,I(A_t; Z_{t-3} | X_t, Z_{t-2:t-1});
    \quad \qquad \dots \qquad \qquad
    M_{t-1} \defeq \,I(A_t; Z_{1} | X_t, Z_{2:t-1}).
\end{align}
Each entry $M_i$ of $M$ quantifies how much additional information about $A_t$
can be gained when considering history of length $i$ instead of only $i-1$. The
first entry $M_0$ is the amount of information $X_t$ shares with $A_t$.

\subsection{Estimating Mutual Information}
Mutual information and conditional mutual information can be written as
differences of entropy terms,
    $I(A_t;  X_t) = H(A_t) + H(X_t) - H(A_t, X_t)$ and
\begin{align}
    I(A_t; Z_{1:t-k} | X_t, Z_{t-k+1:t-1}) 
    = & H(A_t, X_t, Z_{t-k+1:t-1})
    + H(X_t, Z_{t-k:t-1})\\
    & - H(A_t, X_t, Z_{t-k:t-1})
    - H(X_t, Z_{t-k+1:t-1}),
\end{align}
where the entropy of multiple random variables is defined by the entropy of
their joint distribution. We can therefore estimate $M_i$ by estimating the
individual entropy terms.

We use the entropy estimator by
\citet{Grassberger2003} due to its simplicity and computational efficiency;
alternatives~(plug-in, \citet{Nemenman2002,Hausser2009}) yielded very similar results in our
experimental evaluations. The \citet{Grassberger2003} estimate of the entropy
in nats of a random quantity $Y$ of which we have seen $k$ different values,
each $n_1, \dots, n_k$ times, is~\citep{nowozin2012}
    $
    \hat H(Y) = \ln(N) - \frac{1}{N} \sum_{i=1}^k G(n_i),
$
where $N = \sum_i n_i$ is the total number of samples and 
    $
    G(n) = \psi(n) + \frac{(-1)^{n}}{2} \left( \psi\left(\frac {n+1}{2} \right) - \psi \left( \frac n 2 \right) \right),
$
\label{sec:estimation}
with $\psi$ being the \emph{digamma function}.

\subsection{Test of Significance}
\label{sec:permtest}
In practice the sample size available for estimating the conditional mutual
information is limited.
Therefore, our estimates will be affected both by bias and statistical variation.
To prevent invalid conclusions due to bias and variation we use a simple permutation test as follows.
We take the original set of samples $\{(Z_{t-i:t-1}, X_t, A_t)\}$ for
estimating the conditional mutual information $\hat M_i$ and replace the last
action by sampling a new action $\tilde A_t \sim \hat p(A_t)$ from the
empirical marginal of $A_t$. We then compute the conditional mutual information $\tilde M_i$ 
w.r.t. the modified samples $\{(Z_{t-i:t-1}, X_t, \tilde A_t)\}$. This
action-resampling process is repeated $100$ times to obtain the ordered sequence $\tilde M_i^{(1)},
\dots \tilde M_i^{(100)}$ with $\tilde M_i^{(j)} \leq \tilde M_i^{(j+1)}$ for
all $j$. We consider memory use significant if $\hat M_i$ is above the $95\%$
percentile of this set, i.e., $\hat M_i \geq \tilde M_i^{(95)}$.

\vspace{-2mm}
\section{Experimental Results}
\label{sec:experiments}
We trained Deep Q-Network policies
for $50$ million time steps on 49 Atari games. The network structure as well as
all learning parameters have been chosen to match the setting by \citet{Mnih2015}.
Each policy chooses with probability $\epsilon=0.05$ an available action
uniformly at random and otherwise takes the action that maximizes the learned
Q-function. The Q-function takes as input the last four frames of the games as
$84\times 84$ pixel grayscale images. We can interpret this as the agent having
memory to perfectly store the last 4 observations. We applied our memory lens
method and estimated to what extent each of these observations are actually used
when making decisions in each game.

We recorded $10000$ games played by each of the fully trained 49 policies. For
each policy we used these $10000$ trajectories of length up to at most $10000$
time steps to estimate the memory use.
Since we know that the policies are stationary, we expect their use of memory
to be fairly stationary too.  We therefore did not estimate $M$ for the actions at
each time $t$ individually but aggregated samples for all actions.

The results are shown in Figure~\ref{fig:cmi_atari}.
The top bar plot shows the mutual information estimate $\hat M_0$ of action and
most recent frame (requires no memory) and the plots below show $\hat M_1$, $\hat M_2$ and $\hat
M_3$ for the policies of each game. Only bars are displayed that indicate statistically
significant dependencies (see Section~\ref{sec:permtest}).

\begin{figure}[t]
    \includegraphics[width=\textwidth]{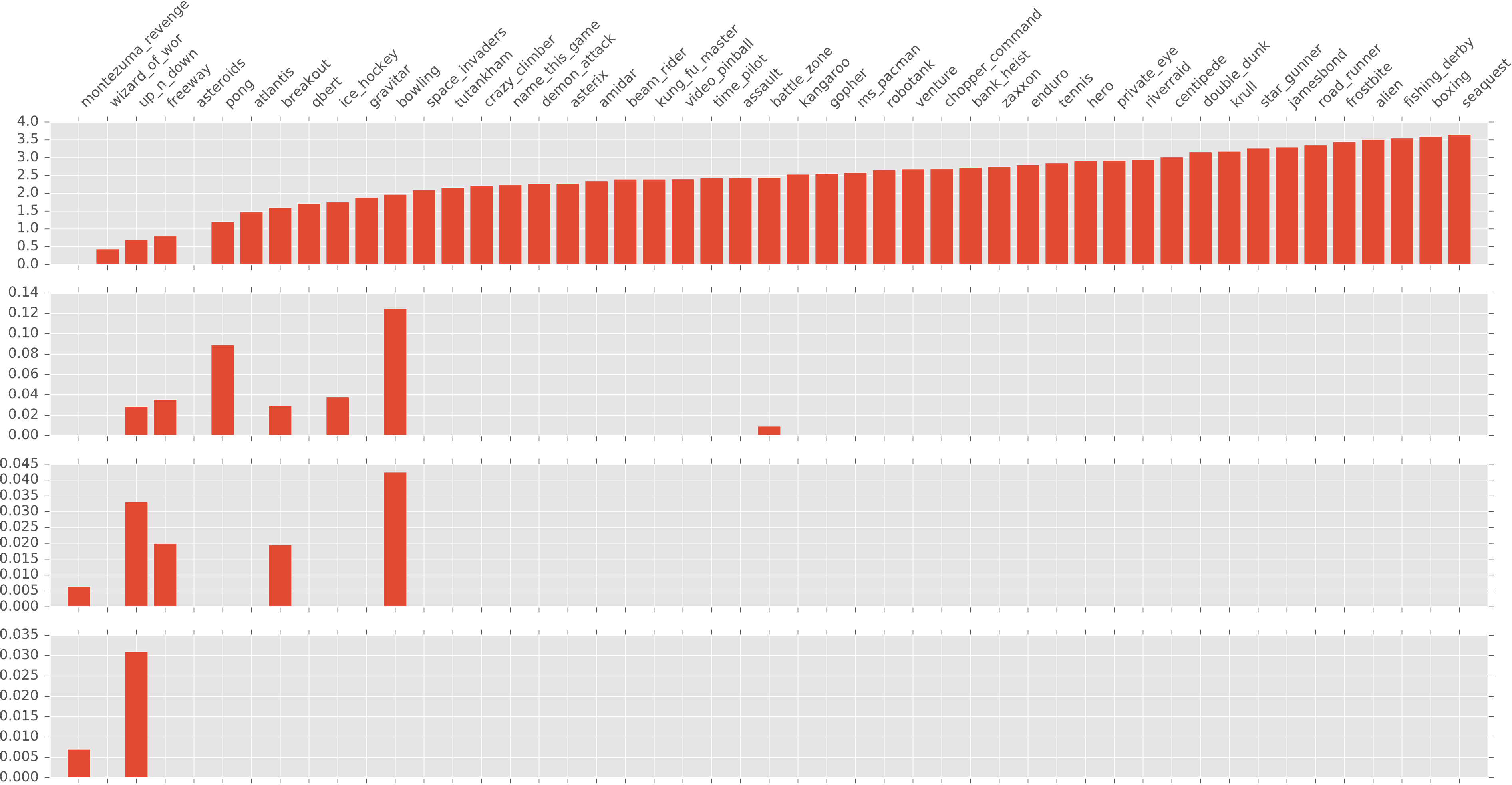}
    \caption{Estimated use of memory by DQN policies on 49 Atari 2600 games. The
    top bar chart shows the mutual information between the chosen action and
    the current game screen; The three plots below visualize the additional
    information used by the policies from previous frames in temporally decreasing
    order. 
}
    \label{fig:cmi_atari}
\end{figure}

\vspace{-3mm}
\section{Formalization of Memory}
\label{sec:theory}
In this section we provide a more formal definition of amount of memory
required to implement an agent's behavior and relate the quantities $M$
estimated by our memory lens approach to it. For the sake of conciseness, we
focus on finite-horizon episodic decision problems with a fixed horizon of $H$.
A single episode $\xi$ is then an element of $\mathcal Z^H$. We use the short-hand notation $[H] = \{1, 2, 3, \dots, H\}$.
Assume an abstract model of memory where the state of memory can take $K \in \mathbb N$ different values.
Following the formalization by \citet{Chatterjee2010}, we define:
\begin{defn}
A \emph{memory function} $g : [H] \times \mathcal Z \times [K] \rightarrow [K]$
maps for each time step the previous memory and current observations to a new
memory configuration.
The policy $\pi$ of an agent can be implemented with $K$ memory if and only if
there is a memory function $g$ so that
   $ P(A_t | X_t, Y_{t-1}) = P(A_t | X_t, Z_{1:t-1})$
for all $t \in [H]$ where $Y_t$ denotes the memory configuration after time
$t$. That is $Y_t = g(t, Z_t, Y_{t-1})$ for $t \in [H]$ and $Y_0 = 1$. We say
that $g$ is a memory function for $\pi$ in this case and denote the
set of memory functions with capacity $K$ for $\pi$ by $\mathcal M_{M,
\pi}$.
    \label{def:memfun}
\end{defn}
Note that this abstract model of memory is very general and, for example,
recurrent neural networks can be considered a direct implementation of it.
We can
formally define the minimum amount of memory required by a policy as
\begin{defn}
    The memory capacity $\mathcal C(\pi)$ of a policy $\pi$ is the smallest
    amount of memory capacity required to reproduce this policy. Formally, it
    is    
        $\mathcal C(\pi) =  \min \{ K \in \mathbb N \, : \, \mathcal M_{K, \pi} \neq \emptyset \}$.
\end{defn}
If the distribution $P$ from which the episodes are sampled from is known,
$\mathcal C(\pi)$ can be computed in finite time (there are only finitely many
memory functions). 
This implies that in principle $\mathcal C(\pi)$ can be
estimated just from observations by estimating $P$ and then computing $\mathcal
C(\pi)$. However, to determine $\mathcal C(\pi)$ one has to perform tests on
equality of conditional probabilities (condition in
Def.~\ref{def:memfun}). Each conditional probability has to be
estimated accurately which requires infeasibly many samples for any problem
beyond simple toy settings.

Instead, the mutual information quantities $M$ used in our method are much easier to estimate.

\subsection{$\sum _{i>0}M_i$ is a lower bound on $\log \mathcal C(\pi)$}

The next corollary states that for any $t$ the sum of all conditional
mutual quantities $M_1, \dots, M_{t-1}$ estimated by our memory lens approach (excluding $M_0$) is a lower bound on
the log-memory capacity.
\begin{cor}
    For all $t \in [H]$, 
        $\sum_{i=1}^{t-1} M_i = \sum_{i=1}^{t-1} I(A_t; Z_{i} | X_t, Z_{i+1:t-1}) \leq \log \mathcal C(\pi)$.
\end{cor}
\vspace{-2mm}
Instead of proving this corollary directly, we show a stronger version in the
theorem below.
This theorem allows to restrict the measure $P$ to an event that can be decided
based on the history and shows that this still results in a valid lower bound
on $\log \mathcal C(\pi)$. In some scenarios one might have prior intuition
when an agent uses memory to make a decision. One can then restrict the measure
to the event $E$ where one expects the memory use will be higher than in
$\Sigma \setminus E$ and obtain a possibly tighter lower bound on $\log
\mathcal C(\pi)$. 
\begin{thm}
    Let $k < t$ and $E \in \sigma(Z_{1:t-1}, X_t)$ be an event in the
    sigma-field generated by the history up to time $t-1$ and the observation
    at time $t$. Denote by $P_E$ the probability measure that restricts $P$ to
    $E$. Then
    $    
    I_E(A_t; Z_{1:k}  | X_t, Z_{k+1:t-1}) \leq \min_{g \in \mathcal M_{\infty, \pi}} \log | g(k, \mathcal Z, [H]) | \leq \log C(\pi),
$    
    where $I_E$ denotes the (conditional) mutual information with respect to $P_E$.
    \label{thm:lb}
\end{thm}
\begin{proof}
    Since $Y_t$ is a function of $Y_k$ and $Z_{k+1:t}$ for any $k < t$, the
    generated sigma-fields satisfy $\sigma(Z_{1:t-1}) \supseteq \sigma(Y_k,
    Z_{k+1:t-1}) \supseteq \sigma(Y_{t-1})$. From $P(A_t | X_t, Y_{t-1}) =
    P(A_t | X_t, Z_{1:t-1})$, it follows that $P_E(A_t | X_t, Y_{t-1}) =
    P_E(A_t | X_t, Z_{1:t-1})$ and hence $P_E(A_t | X_t, Y_{k},
    Z_{k+1:t-1}) = P_E(A_t | X_t, Z_{1:t-1})$.
    We can then equivalently write
    \begin{align*}
        \frac{P_E(A_t , X_t, Y_{k}, Z_{k+1:t-1}) P_E(X_t, Z_{k+1:t-1})}{ P_E(X_t, Z_{k+1:t-1}, Y_k) P_E(X_t, Z_{k+1:t-1}, A_t)}
        = \frac{P_E(A_t, X_t, Z_{1:t-1}) P_E(X_t, Z_{k+1:t-1})}{P_E(X_t, Z_{1:t-1}) P_E(X_t, Z_{k+1:t-1}, A_t)}
    \end{align*}
    which implies that 
        $I_E(A_t; Z_{1:k} | X_t, Z_{k+1:t-1}) = I_E(A_t; Y_k | X_t, Z_{k+1:t-1})$.
    We then can bound the conditional mutual information using basic properties of entropies as
    \begin{align*}
        &I_E(A_t; Z_{1:k} | X_t, Z_{k+1:t-1})
        = I_E(A_t; Y_k | X_t, Z_{k+1:t-1}) \\
        =\,\,&  H_E(Y_k | X_t, Z_{k+1:t-1}) - H_E(Y_k |A_t, X_t, Z_{k+1:t-1}) 
        \leq H_E(Y_k | X_t, Z_{k+1:t-1}) \\
        \leq\,\, & H_E(Y_k)
       \leq \log | Y_k(E) | \leq \log | Y_k(\Omega) |.
     \end{align*}
     \vspace{-.4cm}
\end{proof}

\section{Related Work}
\citet{Papapetrou2016} perform
statistical tests based on conditional mutual information to identify the order of Markov chains.
This is similar to our method but we are only interested on parts of the stochastic process, namely the
agent's actions.
In the work of \citet{Tishby2011} mutual information is used as part of an
optimization objective for policies. Instead of just for maximum cumulative
reward, they optimize for the best trade-off between information processing
cost and cumulative reward. 

Our definition of a \emph{memory function} matches the one by
\citet{Chatterjee2010}.  While we are concerned with the analysis of our
method, they use memory functions for asymptotic theoretical analysis of memory
required to solve POMDPs with parity objectives.
In the abstract model of memory in Section~\ref{sec:theory}, the memory state is
essentially a sufficient statistic summarizing all information from the past
relevant for any action in the future. Sufficient statistics for general stochastic
processes are discussed by
\citet{Shalizi2001a} introducing the concept of $\epsilon$-machines. Unlike $\epsilon$-machines, we require memory to be
updated recursively and we are only concerned with the predictive power
regarding future actions.

\vspace{-3mm}
\section{Conclusion}
\label{sec:conclusion}

In this paper, we have proposed an approach for analyzing memory
use of an agent that interacts with an environment. 
We have provided both a theoretical foundation of our method and demonstrated
its effectiveness in an analysis of state-of-the-art DQN policies playing Atari
games.
Our treatment of memory usage in agents opens up a wide range of
directions for follow-up work. First, our method assumes discrete observation
and action spaces. The key challenge in extending to continuous space is the
need to efficiently compute mutual information of high-dimensional, continuous
observations. A promising avenue is to explore approximations that have been
developed in domains such as neural coding, such as variational information
maximization \citep{agakov2004variational}.

Another interesting question to explore is whether the estimate of memory use
by a policy can be improved when the environment can be controlled actively.
That is, the behavior of an agent can actively be explored by manipulating the
observations and rewards the agent receives. The task of identifying the events
in which the agents requires the maximum amount of memory by manipulating its
observations could possibly be set up as a reinforcement learning task itself.
Further, estimating the amount of memory necessary to solve a task could
potentially be used as an empirical measure for difficulty of sequential
decision making tasks. Many real-world tasks require high-level reasoning with
longer-term memory. While current reinforcement learning algorithms still
mostly fail to achieve reasonable performance on such tasks, often experts,
e.g. humans, can be observed when solving the task.  Analyzing their memory use
could not only give an indication of how difficult a task is but also possibly
inform the design of successful reinforcement learning architectures. 

\bibliographystyle{unsrtnat-nourl}
\bibliography{cdann_mendeley,manual}

\end{document}